\documentclass{article}

\usepackage{times}
\usepackage{graphicx} %
\usepackage{subfigure} 

\usepackage{natbib}

\usepackage{algorithm}
\usepackage{algorithmic}
\usepackage{hyperref}
\usepackage{appendix}

\usepackage[accepted]{nonanon_icml2018}

\usepackage{graphicx}
\usepackage{amsmath,amssymb,amsthm,bm,enumitem}
\usepackage{latexsym,color,verbatim,minipage-marginpar,caption,multirow}
\usepackage{multicol}
\usepackage{rotating}

\usepackage{pifont}%
\newcommand{\cmark}{\ding{51}}%
\newcommand{\xmark}{\ding{55}}%

\newcommand{\eat}[1]{}

\floatstyle{ruled}
\newfloat{algorithm}{tb}{loa}
\floatname{algorithm}{Algorithm}
\newtheorem{theorem}{Theorem}

\DeclareFontFamily{U}{mathx}{\hyphenchar\font45}
\DeclareFontShape{U}{mathx}{m}{n}{
      <5> <6> <7> <8> <9> <10>
      <10.95> <12> <14.4> <17.28> <20.74> <24.88>
      mathx10
      }{}
\DeclareSymbolFont{mathx}{U}{mathx}{m}{n}
\DeclareMathSymbol{\bigtimes}{1}{mathx}{"91}

\def\E{{\ensuremath{\mathbb E}}}

\long\def\comment#1{}

\newcommand{\bcal}{\ensuremath{\mathcal B}}
\newcommand{\ccal}{\ensuremath{\mathcal C}}

\newcommand{\fcal}{\ensuremath{\mathcal F}}

\newcommand{\hcal}{\ensuremath{\mathcal H}}

\newcommand{\mcal}{\ensuremath{\mathcal M}}
\newcommand{\ncal}{\ensuremath{\mathcal N}}

\usepackage{colortbl}
\icmltitlerunning{On Nonlinear Dimensionality Reduction, Linear Smoothing and Autoencoding}

\begin{document} 

\twocolumn[
\icmltitle{On Nonlinear Dimensionality Reduction, Linear Smoothing and Autoencoding}

\icmlsetsymbol{equal}{*}

\begin{icmlauthorlist}
\icmlauthor{Daniel Ting}{tableau}
\icmlauthor{Michael Jordan}{berkeley}
\end{icmlauthorlist}

\icmlaffiliation{tableau}{Tableau Software, Seattle, WA, USA}
\icmlaffiliation{berkeley}{University of California, Berkeley, CA, USA}

\icmlcorrespondingauthor{Daniel Ting}{dting@tableau.com}
\icmlcorrespondingauthor{Michael Jordan}{jordan@stat.berkeley.edu}

\icmlkeywords{manifold, nonlinear dimensionality reduction, Laplacian, LLE, LTSA}

\vskip 0.3in
]

\printAffiliationsAndNotice{}  %

\begin{abstract} 
We develop theory for nonlinear dimensionality reduction (NLDR).  A number of
NLDR methods have been developed, but there is limited understanding of how 
these methods work and the relationships between them.  There is limited basis 
for using existing NLDR theory for deriving new algorithms.  We provide a novel
framework for analysis of NLDR via a connection to the statistical theory of
 linear smoothers. This allows us to both understand existing methods 
and derive new ones.  We use this connection to smoothing to show that 
asymptotically, existing NLDR methods correspond to discrete approximations 
of the solutions of sets of differential equations given a boundary condition. 
In particular, we can characterize many existing methods in terms of just three 
limiting differential operators and boundary conditions. Our theory also provides
a way to assert that one method is preferable to another; indeed, we show 
Local Tangent Space Alignment is superior within a class of methods that assume
a global coordinate chart defines an isometric embedding of the manifold.  
\end{abstract} 

\vspace{-0.2cm}
\section{Introduction}
One of the major open problems in machine learning involves the development
of theoretically-sound methods for identifying and exploiting the low-dimensional 
structures that are often present in high-dimensional data.  Such methodology
would have applications not only in supervised learning, but also in visualization 
and nonlinear dimensionality reduction, semi-supervised learning, and manifold 
regularization.

Some initial steps have been made in this direction over the years under the
rubric of ``manifold learning methods.''  These nonlinear dimension reduction 
(NLDR) methods have permitted interesting theoretical analysis and allowed
the field to move beyond linear dimension reduction.  But the theoretical results
have fallen short of providing characterizations of the overall scope of the
problem, including the similarities and differences of existing methods and
their respective advantages.  Consider, for example, the classical problem 
of finding nonlinear embeddings of a Swiss roll with hole, shown in 
Figure \ref{fig:swisshole}.  Of the methods shown, Local Tangent Space 
Alignment is clearly the best at recovering the underlying planar structure 
of the manifold, but there is no existing theoretical explanation of this
fact.  Nor are there answers to the natural question of whether there are 
scenarios in which other methods would be better, why other methods perform 
worse, or whether the deficiencies can be corrected.  The current paper aims
to tackle some of these problems.  Not only do we provide new characterizations of NLDR methods, but by correcting deficiencies of existing methods we are able to propose new methods in Laplace-Beltrami approximation theory.

We analyze the general class of manifold learning methods that we refer to
as \emph{local, spectral methods}.  These methods construct a matrix using only 
information in local neighborhoods and take a spectral decomposition to find a nonlinear embedding.  This framework includes the 
commonly used methods:
 Laplacian Eigenmaps (LE) \cite{belkin2003laplacian},
 Local Linear Embedding (LLE) \cite{roweis00LLE},
Hessian LLE (HLLE) \cite{HessEig},
Local Tangent Space Alignment (LTSA) \cite{zhang2004principal}, and
Diffusion Maps (DM) \cite{coifman2006diffusion}.
We also consider several recent improvements to these classical methods, 
including Low-Dimensional Representation-LLE \cite{goldberg2008ldrlle},
Modified-LLE \cite{zhang2007mlle}, and MALLER \cite{cheng2013local}.
Outside of our scope are global methods that construct dense matrices 
that encode global pairwise information; these include multidimensional 
scaling, principal components analysis, isomap \cite{tenenbaum2000global}, 
and maximum variance unfolding \cite{weinberger2006unsupervised}.  We 
note in passing that these global methods have serious practical limitations,
either in terms of a strong linearity assumption or computational intractability.

Our general approach proceeds by showing that the embeddings for convergent 
local methods are solutions of differential equations, with a set of boundary 
conditions induced by the method.  The treatment of the boundary conditions
is the critical novelty of our approach.  This approach allows us to categorize 
the methods considered into three classes of differential operators and 
accompanying boundary conditions, as shown in Table \ref{tbl:categorization}. 
We are also able to delineate properties that allow comparisons of methods
within a class; see Table \ref{tbl:properties}.  These theoretical results
allow us to, for example, conclude that when the goal is to find an 
isometry-preserving, global coordinate chart, LTSA is best among the
classical methods considered. We can also show that HLLE belongs to the 
same class and converges to the same limit. Hence, they can be
used exchangeably as smoothness penalties.  We also improve existing 
Laplace-Beltrami approximations. In particular, we give a Laplace-Beltrami 
approximation that is consistent as a smoothing penalty even when boundary 
conditions are not satisfied.

Our analysis is based on the following two insights.  First, matrices obtained
from local, spectral methods can be seen as an operator that returns the bias 
from smoothing a function.  This allows us to make a connection to the theory
of linear smoothing in statistics.  Second, as the neighborhood sizes for each 
local method decrease, we obtain a linear operator that evaluates infinitesimal 
differences in that neighborhood. In other words, we obtain convergence to a 
differential operator. Thus, the asymptotic behavior of a procedure can be 
characterized by the corresponding differential operator in the interior of 
the manifold and (crucially) by the boundary bias in the linear smoother.

\setlength{\belowcaptionskip}{-0.1cm}
\begin{table}
	\begin{tabular}{c|ll}		 
		\multirow{2}{*}{\parbox{1.5cm}{Boundary Condition}} & 
		\multicolumn{2}{c}{2nd order penalty} \\  	
		& \multicolumn{1}{c}{$\E \;Tr(\hcal f)^2$} &
		\multicolumn{1}{c}{$\E\; \|\hcal f\|^2$} \\
		\hline
		\cellcolor[gray]{0.8} { None} & { Coefficient Laplacian} &  \\ \cline{2-3}
		\cellcolor[gray]{0.8}\multirow{1}{*}{${f = 0}$}  & { LLR-Laplacian} &  \\ \cline{2-3}
		\multirow{2}{*}{$\frac{\partial f}{\partial \eta} = 0$}  & Diffusion maps & \\ 
		&  Laplacian Eigenmaps & \\ \cline{2-3}
		 \cellcolor[gray]{0.8} & (LDR-, m-)LLE$^*$ & HLLE \\		
		\cellcolor[gray]{0.8} & LDR-LLE+  & LTSA \\
		\multirow{-3}{*}{\cellcolor[gray]{0.8}$\frac{\partial^2 f}{ \partial \eta^2} = \beta \Delta f$} & LLR-Laplacian &		
	\end{tabular}
\vspace{-0.0cm}
	\caption{Categorization of NLDR methods by their limit operator and boundary conditions. Within each category, methods with better properties appear closer to the bottom. Newly identified boundary conditions are highlighted in grey. LLE and existing variants have a dependence on the manifold curvature which makes the induced second-order penalty different from $\E\;Tr(\hcal f)^2$.}
	\label{tbl:categorization}
\end{table}

\setlength{\belowcaptionskip}{-0.05cm}
\begin{table}
	\begin{tabular}{c|ccc}		 
		Method & $L \succeq 0$ & \parbox{1.3cm}{Order of smoother} & \parbox{1.3cm}{Converges \& stable}\\ \hline
		Diffusion Maps & \xmark & $0^{th}$ & \cmark\\
		Laplacian Eigenmaps & \cmark & $0^{th}$ & \cmark\\ \hline
		LLE & \xmark & $2^{nd}$ & \xmark \\	
		LDR-LLE & \xmark & $2^{nd}$ & \cellcolor[gray]{0.8}  \xmark \\	
		\cellcolor[gray]{0.8} LDR-LLE+ & \xmark & $1^{st}$ & \cellcolor[gray]{0.8}  \cmark \\	
		LLR-Laplacian & \xmark & $1^{st}$ & \cmark \\ \hline	
		HLLE & \cmark & $2^{nd}$ & \cellcolor[gray]{0.8}  \cmark \\	
		LTSA & \cmark & $1^{st}$ & \cellcolor[gray]{0.8}  \cmark \\ \hline
		\cellcolor[gray]{0.8} Coefficient Laplacian & \cmark & $1^{st}$ & \cmark 
	\end{tabular}
	\caption{Properties of NLDR methods. Non-convergent or unstable methods may converge if regularization is added. Lower order smoothers generally have less variance. New methods and convergence results are highlighted in grey. }
	\label{tbl:properties}
\end{table}

\vspace{-0.1cm}
\section{Preliminaries}
\label{sec:preliminaries}
We begin by introducing some basic mathematical concepts and providing a 
high-level overview of the relationship between NLDR and linear smoothing.

Local, spectral methods share a construction paradigm:
\vspace{-0.1cm}
\begin{enumerate}[nolistsep]
	\item Choose a neighborhood for each point $x_i$;
	\item Construct a matrix $L$ with $L_{ij} \neq 0$ only if $x_j$ is a neighbor of $x_i$;
	\item Obtain a nonlinear embedding from an eigen- or singular value decomposition of $L$.
\end{enumerate}
Step 2 for constructing the matrix $L$ can be seen to be equivalent to constructing a linear smoother. 
We use this equivalence to compare existing NLDR methods and generate new ones by examining the asymptotic bias of the smoother on the interior and boundary of the manifold. In particular, we show that whenever the neighborhoods shrink in Step 1 as the number of points $n \to \infty$, the matrix $L$ converges to a differential operator on the interior of the manifold. Nonlinear dimensionality reduction methods thus solve a eigenproblem involving this differential operator. The solutions of the eigenproblem on a compact manifold are typically only well-separated and identifiable in the presence of boundary conditions. If $L$ is symmetric and the smoother's bias has a different asymptotic rate on the boundary than the interior, then it imposes a corresponding boundary condition for the eigenproblem. 
If $L$ is not symmetric, then the boundary conditions can be determined from the boundary bias for $L$ and its adjoint $L^T$.

\begin{figure}
\includegraphics[height=1.7in, width=3.4in]{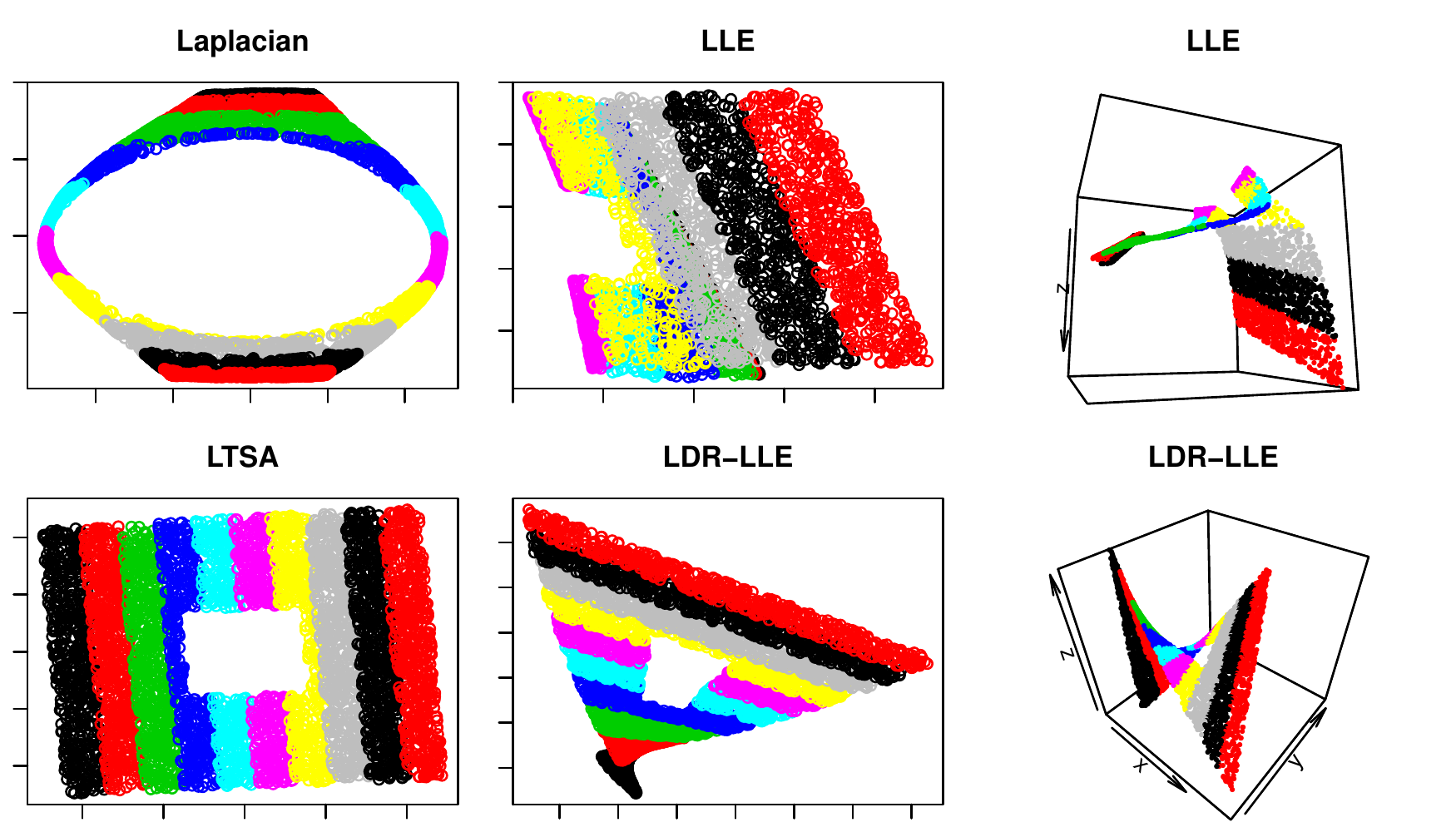} 
	\caption{The behavior of methods on the Swiss roll with hole are shown. LTSA yields the best reconstruction. The boundary effects of the graph Laplacian distort the manifold. LLE's embedding often contains kinks. The LLR-Laplacian and LDR-LLE produce a 3D embedding which is nearly identical up to rotations. This embedding approximates the functions $x_1, x_2, x_1 x_2$.  LLR-LLE was the better of the two at picking out a 2D embedding. }
	\label{fig:swisshole}
\end{figure}

\vspace{-0.1cm}
\subsection{Linear smoothers}
Consider the standard regression problem for data generated by the model $Y_i = f(X_i) + \epsilon_i$ where $\epsilon_i$ is a zero-mean error term.
A regression model is a \textit{linear smoother} if 
\vspace{-0.05cm}
\begin{align}
\hat{\mathbf{y}} &= S(\mathbf{x})\mathbf{y}
\end{align}
\vspace{-0.05cm}
and the smoothing matrix $S(\mathbf{x})$ does not depend on the outcomes $\mathbf{y}$.
Examples of linear smoothers include linear least squares regression on a matrix $X$, 
kernel and local polynomial regression, and ridge regression. 

In the regression problem, one wishes to infer the conditional mean function $f$ given noisy observations $Y$. For manifold learning problems, the goal is to learn noiseless global coordinates $\Phi_i = f(X_i)$ but the outcomes $\Phi$ are never observed. 
In both linear smoothing and manifold learning,
the smoothing matrix makes no use of the response and simply maps rough functions into a space of smoother functions. The embedding that is learned by NLDR methods is a space of functions that are best reconstructed by the smoother. In other words, the smoother defines a form of autoencoder and the learned space is a space of likely outputs from the autoencoder.

\vspace{-0.1cm}
\subsection{NLDR constructions of linear smoothers}
\label{sec:nldr smoothers intro}
It is a simple but crucial fact that each of the NLDR methods that
we consider construct or approximate a matrix $L = G(I - S)$ where 
$G$ is a diagonal matrix and $S$ is a linear smoother.  Thus, $L$ 
measures the bias of the prediction $S f$ weighted by $G$.

For example, Diffusion Maps, with a Gaussian kernel constructs the 
Nadaraya-Watson smoother $S = D^{-1} K$, where $D$ is the diagonal 
degree matrix and $K$ is the kernel matrix. The constructed embedding 
consists of the right singular vectors corresponding to the smallest 
singular values of $L^{DM} = I - S$.  
Laplacian Eigenmaps, using an unnormalized graph Laplacian, differ 
only by the reweighting $L^{LE} = D L^{DM} = D(I - S)$, which ensures 
that the matrix $L^{LE}$ is positive semidefinite.
\vspace{-0.1cm}
\subsection{Assumptions}
The usual setting for linear smoothing is Euclidean space. To account for the manifold setting, we must make some regularity assumptions and demonstrate that calculations in Euclidean space  approximate calculations on the manifold sufficiently closely.

We consider a smooth, compact $m$-dimensional Riemannian manifold $\mcal$ with smooth boundary
embedded in $\mathbb{R}^d$. 
From this manifold, points $\{X_1, \ldots, X_n\} \in \mathbb{R}^d$ are drawn from a uniform distribution on the manifold. 
This uniformity assumption can be weakened to admit a smooth density and obtain density weighted asymptotic results. 

In the continuous case, we consider neighborhoods $\ncal_x = B(x, h)$ for a given bandwidth $h$ where distance is Euclidean distance in the ambient space. When appropriate, we will take $h \to 0$ as $n \to \infty$. In the discrete case, denote by $\ncal_i$ the indices for neighbors of $X_i$ in the point cloud. This neighborhood construction is only mildly restrictive since kNN neighborhoods are asymptotically equivalent to $\epsilon$ neighborhoods in the interior of the manifold when points are sampled uniformly and the radius of the neighborhoods is $\Theta(h)$ on the boundary. 

\vspace{-0.1cm}
\subsection{Local coordinates}
\label{sec:local coordinates}
Most existing methods and the ones we construct require access to a local coordinate system with manifold dimension $m$ for each point.
This coordinate system is estimated using a local PCA or SVD at each point. Let $X_{\ncal}$ be the $|\ncal| \times d$ matrix of ambient space coordinates for points in neighborhood $\ncal$ centered on $x$ and $\tilde{X}_{\ncal}$ be the corresponding matrix centered on $x$.
The $m$ top right singular vectors
\begin{align}
\tilde{X}_{\ncal} &= U_{\ncal} \Lambda V_{\ncal}^T
\end{align}
give estimated tangent vectors at $x$.
The top rescaled left singular vectors, $\tau_i = \lambda_i U_{\ncal, i}$, project 
points in $\ncal_i$ to the tangent space.
The normal coordinates $u_i$, corresponding to the geodesics traced out by the tangent vectors $V_{\ncal,i}$, agree closely with these tangent space coordinates. 
Specifically, by Lemma 6 in \citet{coifman2006diffusion}, 
$u_i = \tau_i + q_3(\mathbf{t}) +  O(h^4)$, where $q_3$ is a homogeneous polynomial of degree three whenever the coordinates are in a unit ball of radius $h$. 
We adopt the convention that $u_i(y)$ refers to the normal coordinates at $x \in \mcal$ for $y \in \ncal_x$.

Our results rely on integration in normal coordinates for a ball of radius $h$ in the distance metric of the ambient space. To account for the manifold curvature, the volume form and neighborhood sizes must be accounted for in the integral.
 Lemma 7 in \citet{coifman2006diffusion} further provides a Taylor expansion for the volume form for the Riemannian metric $dV_g(u) = 1 + q_2(u) + O(h^3)$, where $q_2$ is a homogeneous polynomial of degree two. Likewise, distances in the ambient space and normal coordinates differ by $\|y-x\|^2 = \|u(y)\|^2 + \tilde{q}_2(u(y)) + O(h^3)$ where $u(y)$ denotes normal coordinates for $y$ about $x$, $\tilde{q}_2$ is homogenous degree 2, and the distance on the left-hand side 
is with respect to the ambient space.
 Consequently, integrals for any homogeneous polynomial $\acute{q}$ of degree one or two satisfy
 \begin{align}
 \int_{\ncal_x} \acute{q}(u) dV_g(u) &= \int_{B(0,h)} \acute{q}(u) du + O(h^4),
 \end{align}
where the integral on the right represents integration in $\mathbb{R}^m$
and the $O(h^4)$ term hides a smooth function that depends on the curvature 
of the manifold at $x$. 
 
For the purposes of this paper, we do not account for error from estimation 
of the tangent space.  Thus we assume that we have a sufficiently accurate 
estimate of the normal coordinates.  Accounting for this estimation error is
a natural direction for further work.

\section{Analysis of existing methods}
\label{sec:analysis}
Each of the existing NLDR methods construct a matrix $L_n$ from a point cloud of $n$ points with bandwidth $h$. The resulting nonlinear embeddings are obtained from the bottom eigenvectors of $L_n$ or $L_n^T L_n$. 
We consider $L_n$ as a discrete approximation to an operator $\mathbf{L}_h : \fcal \to \fcal$ on functions $\fcal \subset C^{\infty}(\mcal)$.
We examine the limit operator $\mathbf{L}_h$ constructed by each of the NLDR methods for a fixed bandwidth $h$. We show that as the bandwidth $h \to 0$, $\mathbf{L}_h \to \mathbf{L}_0$ where $\mathbf{L}_0$ is a differential operator. The exception is LLE where there is no well defined limit. The stochastic convergence of the empirical constructions $L_n \stackrel{n \to \infty}{\to}  L_h$ is not considered in this paper.

\vspace{-0.1cm}
\subsection{Taylor expansions and local polynomial bases}
As described in Section \ref{sec:nldr smoothers intro}, most existing NLDR 
methods can be expressed as $L_h = D_h(I - S_h)$, where $D_h$ is a multiplication operator corresponding to a diagonal matrix and $S_h$ is a linear smoother. 
Denote by $S_h(x, \cdot)$ the linear function such that $\langle S_h(x, \cdot), f \rangle = (S_h f)(x)$. This function exists by the Riesz representation theorem. We say $S_h$ is \emph{local} when for all $x \in \mcal$, the support of $S_h(x, \cdot)$ is contained in a ball of radius $h$ centered on $x$. We assume that $S_h$ is a bounded 
operator (as unbounded operators are necessarily poor smoothers). 

Consider $S_h$ applied to a function $f \in C^\infty(\mathcal{M})$.
A Taylor expansion of $f$ in a normal coordinates at $x$ gives
\begin{align}
\label{eqn:operator poly}
(S_h f)(x) &= (S_h 1)(x)  f(x)  + \nabla f (x) (S_h (y-x))(x) \nonumber \\ 
&\quad +  Tr\left( \hcal f (x) S_h (y-x)(y-x)^T\right)(x) + \cdots \nonumber \\ 
&\quad + O(h^r),
\end{align}
where the error term holds since 
$f$ has bounded $r^{th}$ derivatives due to the compactness of $\mcal$.
Here $y-x = u$ is a function denoting the normal coordinate map at $x$.
From this it is clear that the asymptotic behavior of $S_h$ in a neighborhood of $x$ is determined by its behavior on a basis of local polynomials. Furthermore, it can well approximated by examining the behavior on only low-degree polynomials.

\vspace{-0.1cm}
\subsection{Convergence to a differential operator}
Of particular interest is the case where $S_n$ operates locally on shrinking neighborhoods. In this case, the following theorem (proven in the supplementary material) is an immediate consequence of applying the sequence of smoothers to the Taylor expansion.
\begin{theorem}
Let $S_n$ be a sequence of linear smoothers where the support of $S_n(x,\cdot)$ is contained in a ball of radius $h_n$.
Further assume that the bias in the residuals $h_n^{-k}(I-S_n)f = o(1)$ for some integer $k$
and all $f \in \ccal \subset C^{\infty}(\mcal)$. Then if $h_n \to 0$ and $h_n^{-k}(I-S_n)$ converges to a bounded linear operator as $n \to \infty$, then it must converge to a differential operator of order at most $k$ on the domain $\ccal$.
\end{theorem}
As a result, each of the existing NLDR methods considered can be seen as discrete approximations for the solutions of differential equations $D f = \lambda f$ given some boundary conditions, where $D$ is some differential operator and $\lambda \in \mathbb{R}$.

\subsection{Laplacian and boundary behavior}
The convergence of Laplacians constructed from point clouds to a weighted Laplace-Beltrami operator has been well-studied by \citet{HeinGraphNormalizations} and \citet{ting2010laplacian}. 
In particular, the spectral convergence of Laplacians constructed using kernels has been studied by \citet{belkin2007convergence}, \citet{ConsistencySC}, and \citet{berry2016consistent}. 
In the presence of a smooth boundary, the eigenproblem for the Diffusion Maps Laplacian that converges to the unweighted Laplace-Beltrami operator has been shown by \citet{coifman2006diffusion} and \citet{singer2016spectral} to impose Neumann boundary conditions. Specifically, it solves the differential equation
\begin{align}
\Delta u &= \lambda u  \quad s.t.\quad
\frac{\partial u}{\partial \eta_x}(x) = 0 \quad x \in \partial \mcal,
\label{eqn:boundary condition}
\end{align}
where $\eta_x$ is the vector normal to the boundary at $x$.
This result is easily extended to other Laplacian constructions.

As the Diffusion Maps operator is the bias operator for a Nadaraya-Watson kernel smoother, and the unnormalized Laplacian is a rescaling of this bias, the Neumann boundary conditions can be derived from existing analyses for kernel smoothers. The boundary bias of the kernel smoother is $c \cdot h \cdot \partial f /\partial \eta_x (x) + o(h)$ for some constant $c$ when $x \in \partial \mcal$ while in the interior it is of order $O(h^2)$ when points are sampled uniformly. This trivially follows from noting that the first moment of a symmetric kernel is zero when the underlying distribution is uniform and that the boundary introduces asymmetry in the direction $\eta_x$.
Hence, the operator must be scaled by $\Theta(h^{-2})$ to obtain a non-trivial limit in the interior, but this scaling yields
divergent behavior at the boundary when the Neumann boundary condition $\partial f /\partial \eta_x (x) = 0$ is not met. Since eigenfunctions cannot exhibit this divergent behavior, they must satisfy the boundary conditions.
The significance of this result is given by the following simple theorem.
\begin{theorem}
	Let $L$ be an operator imposing the Neumann boundary condition $\partial f /\partial \eta_x (x) = 0$. Then, if there exists an isometry-preserving global coordinate chart, a spectral decomposition of $L$ cannot recover the global coordinates, as they are not in the span of the eigenfunctions of $L$.
\end{theorem}

\begin{figure}
	\vspace{-0.1cm}
\hspace*{-0.2cm}\includegraphics[height=1.2in, width=3.3in]{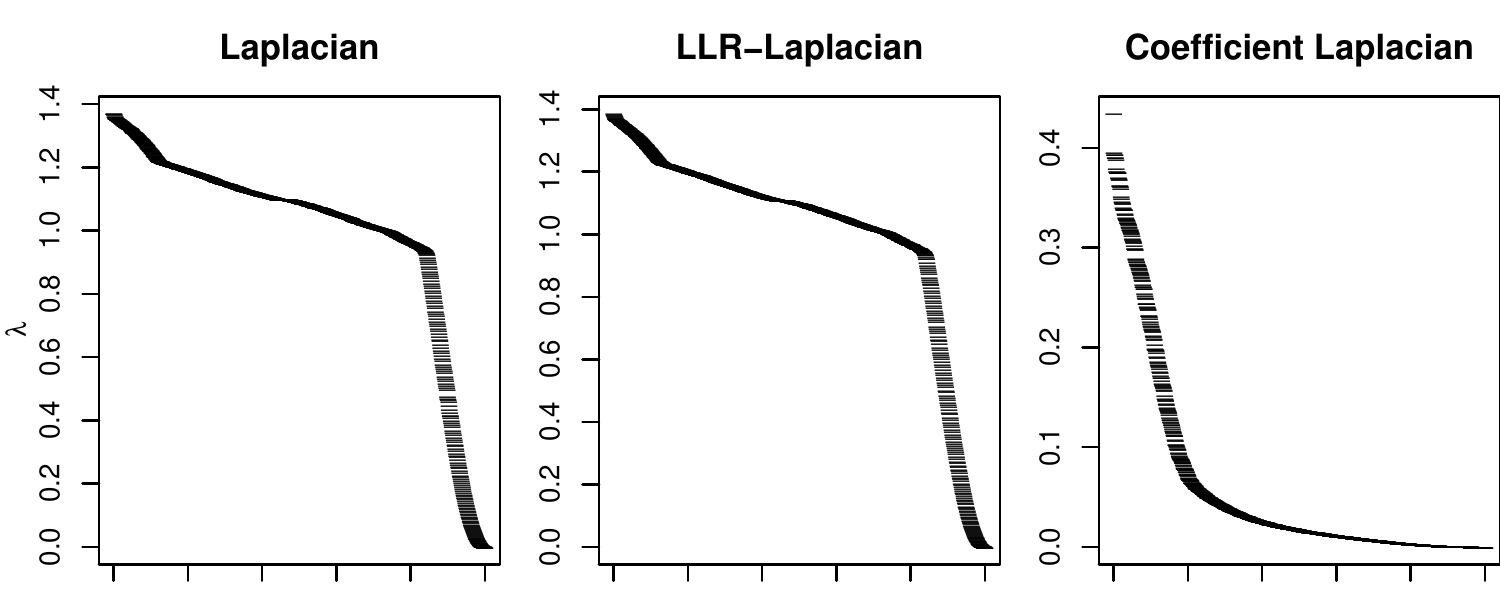} 
	\vspace{-0.1cm}
\hspace*{-0.2cm}\includegraphics[height=0.8in, width=3.3in]{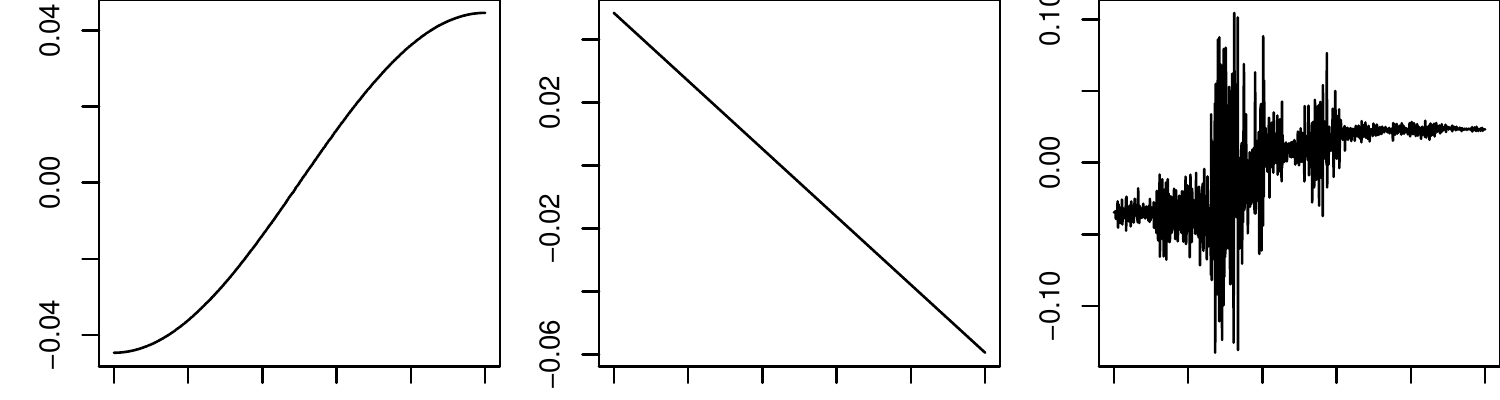}
	\caption{The top row shows spectra for different Laplace-Beltrami approximations on a line segment. The Coefficient Laplacian spectrum is very different as the spectrum is not discrete. The bottom nontrivial eigenvector shows the effect of different boundary conditions. A Neumann condition yields a $cos$ function. A second-order condition yields linear functions. No condition yields useless eigenfunctions. }
	\label{fig:spectra}
\end{figure}

\subsection{Local Linear Embedding}
\label{sec:LLE}
Local Linear Embedding (LLE) solves the boundary bias problem by explicitly canceling out the first two terms in the Taylor expansion.
For a point cloud $\mathbf{X}$, LLE constructs a weight matrix $W$ satisfying $W 1 = 1$ and $W \mathbf{X} = \mathbf{X}$. 
\citet{goldberg2008ldrlle} and \citet{ting2010laplacian} showed that this condition is not sufficient to identify a limit operator. The behavior of LLE depends on a regularization term. We give a more refined analysis of the constrained ridge regression procedure used to generate the weights.

Let $\tilde{\ncal}_i = \ncal_i \backslash \{i\}$ be the neighborhood that excludes $i$ and
$\tilde{X}_{\ncal_i} = X_{\ncal_i} - 1_{|\ncal_i|} X_i$ be the points in $\tilde{\ncal}_i$ centered on $X_i$. To reconstruct $X_i$ from its neighbors, one solves the regression problem
\begin{align}
X_i &= X_{\tilde{\ncal_i}} w_i = X_i + \tilde{X}_{\tilde{\ncal_i}} w_i,
\end{align}
under the constraint that the weights sum to one, $\|w_i\|_1 = 1$.
Adding a ridge regression penalty $\lambda$ and applying Lagrange multipliers gives  $w_i \propto (\tilde{X}_{\tilde{\ncal}_i}\tilde{X}_{\tilde{\ncal}_i}^T + \lambda I)^{-1} 1_{\ncal_i}$.

Consider the singular value decomposition $\tilde{X}_{\tilde{\ncal}_i} = UDV^T$, and let $U_c$ be the orthogonal complement of $U$.
This allows us to rewrite the weights
\begin{align}
w_i &\propto (U(D^2 + \lambda I)^{-1}U^T + \lambda^{-1} U_c U_C^T) 1_{\ncal_i} \\
&\propto 1_{\ncal_i} -  U(I - \lambda (D^2 + \lambda I)^{-1})U^T 1_{\ncal_i}. \label{eqn:LLE weights}
\end{align}
In other words, the weights are the same as the constant weights used by a $h$-ball Laplacian but with a correction term.
This correction term in Eq.~\ref{eqn:LLE weights} depends only on the curvature of the manifold. One can see this by noting that the top $m$ left singular vectors correspond to the points expressed in normal coordinates around $X_i$. As the neighborhood is symmetric around $X_i$, it follows that $U_{1:m}^T 1 \approx 0$. The remaining left singular vectors correspond to directions of curvature of the manifold. 
From this, is it easy to see that LLE with $h$-ball neighborhoods is equivalent to $\Delta_h - C$ in the interior of the manifold where $\Delta_h$ is the $h$-ball Laplacian and $C$ is some correction term that reduces the penalty on the curvature of a function when it matches the curvature of the manifold.

\subsection{Hessian LLE}
Hessian LLE performs a local quadratic regression in each neighborhood. It then discards the locally linear portion in order to construct a functional that penalizes second derivatives.
The Taylor expansion in Eq.~\ref{eqn:operator poly}
gives that the entries of the Hessian are coefficients of a local quadratic polynomial regression.  

Given a basis $Z_{\ncal_i}$ of monomials of degree $\leq 2$ in local coordinates for each point in neighborhood $\ncal_i$
where the first $m+1$ basis elements span linear and constant functions, the entries in the estimated Hessian $\hcal(f)_{X_i}$ are $\Pi(Z_{\ncal_i}^TZ_{\ncal_i})^{-1} Z_{\ncal_i} f$
where $\Pi$ is the projection onto the $m+2$ through $1+m + m(m-1)/2$ coordinates. 

Rather than directly estimating the Hessian, HLLE performs the following approximation. Orthonormalize the basis of monomials using Gram-Schmidt and discard the first $m+1$ vectors to obtain $\tilde{Z}_{\ncal_i}$. 
By orthonormality, the regression coefficients in this basis are simply the inner products
$\tilde{Z}_{\ncal_i}^T f$. 
We note that the approximation only recovers the Hessian in the interior of the manifold and when the underlying sampling distribution is uniform.

Given Hessian estimates $\hat{\beta}_i = \tilde{Z}_{\ncal_i} f$ for each point, the sum of their Frobenius norms is easily expressed as 
\begin{align*}
\displaystyle \sum_i \|\hcal f(x_i)\|^2 &\approx 
\sum_i f^T \tilde{Z}_{\ncal_i}^T  \tilde{Z}_{\ncal_i} f %
= f^T \left(\sum_i Q^{HLLE}_i\right) f.
\end{align*}

\vspace{-0.1cm}
\subsection{Local Tangent Space Alignment}
Local tangent space alignment (LTSA) computes a set of global coordinates that can be aligned to local normal coordinates via local affine transformations.

Given neighbors of $X_i$ expressed in local coordinates $U_{\ncal_i} \in \mathbb{R}^{|\ncal_i| \times m}$,
LTSA finds global coordinates $\mathbf{Y}$ and affine alignment transformations $A_i$ that minimize the difference between the aligned global and local coordinates, $J_i(\mathbf{Y}, \mathbf{A}) = \|C Y_{\ncal_i} - U_{\ncal_i} A_i\|_F^2$,
where $C = 1 - 11^T / |\ncal_i|$ is a centering matrix. 

Given the global coordinates $Y$, this is a least-squares regression problem for a fixed set of covariates $U_{\ncal_i}$. Thus, the best linear predictors for $C Y_{\ncal_i}$ are given by 
$C Y_{\ncal_i} H^{(1)}_{\ncal_i}$ where $H^{(1)}_{\ncal_i}$ is the hat matrix for the covariates $U_{\ncal_i}$. 
The objective can be expressed using local operators $Q^{LTSA}_i$ as
\begin{align}
\min_{\mathbf{A}} \sum_i J_i(\mathbf{Y}, \mathbf{A}) 
&= \sum_i   Y_{\ncal_i}^T C^T (I - H^{(1)}_{\ncal_i}) C Y_{\ncal_i} \nonumber \\
&= Y^T \left( \sum_i  Q^{LTSA}_i \right) Y.
\end{align}

\vspace{-0.2cm}
\section{Equivalence of HLLE and LTSA}
\label{sec:LTSA}
Although LTSA is derived from a very different objective than HLLE, we will show that they are asymptotically equivalent. This greatly strengthens a result in \citet{zhang2018hlleltsa} which showed the equivalence of a modified version of HLLE to LTSA under the restrictive condition that there are exactly $m(m+1)/2 + m - 1$ neighbors.

We show the asymptotic equivalence in two steps. First, we show 
$L_h^{HLLE}$ and $L_h^{LTSA}$ converge to the same differential operator by showing convergence in the weak operator topology.
Then, we give an argument that derives the boundary condition for both methods.

\vspace{-0.1cm}
\subsection{Continuous extension}
Both HLLE and LTSA are composed of the sum of local projection operators $Q_i$ on the point cloud. As we wish to study their limit behavior on smooth functions, we consider the continuous extension to the manifold. To form the continuous extension, the sum is replaced by an integral
\begin{align}
L_h^{LTSA} f &=\frac{\alpha_h}{Vol_h} \int_\mcal Q_y f dV_g(y),
\end{align}
where $V_g$ is the natural volume measure on $\mcal$,
$\alpha_h$ is an appropriate normalizing constant,
$Vol_h$ is the volume of a ball of radius $h$ in $\mathbb{R}^m$,
and $Q_y$ are local operators.
In the case of LTSA, $Q^{LTSA}_y = I_{\ncal_y} - H_y$ where $H_y$ is the projection onto linear functions in the neighborhood $\ncal_y$ and on a local normal coordinate system rather than linear functions on a discrete set of points.

To identify the order of the differential operator that LTSA converges to, we simply need to find the scaling that yields a non-trivial limit.
For the functional $f^T L_h^{LTSA} f$ to non-trivially converge as $h \to 0$, the component terms $f^T Q_y f$ should also non-trivially converge. Since $Q_y$ is a projection onto the space orthogonal to linear functions, this is equivalent to $f^T Q_y^T Q_y f = \| Q_y f\|^2 = O(\alpha_h h^{4} (\|\hcal f_y\|^2 + h))$.
For this to converge non-trivially, one must have $\alpha_h = O(h^{-4})$.
The same argument holds for HLLE.
Thus, both LTSA and HLLE yield fourth-order differential operators in the interior of $\mcal$. 

\vspace{-0.1cm}
\subsection{HLLE and LTSA difference is neglible}
\label{sec:LTSA/HLLE difference}
To show the equivalence of HLLE and LTSA, we show that the remainder term resulting from their difference is the bias term of an even higher order smoother. As such, this bias term goes to zero faster than the HLLE and LTSA bias terms, which yields convergence in the weak operator topology. This gives following theorem. Proof details are given in the supplementary material.
\begin{theorem}
	$L_h^{LTSA} - L_h^{HLLE} \to 0$ as $h\to 0$ in the weak operator topology of $C^\infty(\mcal)$ equipped with the $L_2$ norm.
\end{theorem}

\vspace{-0.1cm}
\subsection{HLLE and LTSA boundary behavior }
To establish that the limit operator has the same eigenfunctions, we must  show that the boundary conditions that those eigenfunctions satisfy also match. We provide a theorem where a single proof identifies the boundary condition for both methods. This boundary condition applies to the second derivatives of a function and admits linear functions, unlike the Neumann boundary condition imposed by graph-Laplacian-based methods.

\begin{theorem}
	$L^{LTSA}_h f(x) \to \infty$ and $L^{HLLE}_h f(x) \to \infty$ as $h\to 0$ for any $x \in \partial \mcal$ and  $f \in C^{\infty}(\mcal)$ unless
\begin{align}
\label{eqn:HLLE boundary}
\frac{\partial^2 f}{\partial \eta_x^2}(x) &= \frac{m+1}{2} (\Delta f)(x) \quad \forall x \in \partial \mcal,
\end{align}
where $\eta_x$ is the tangent vector orthogonal to the boundary.
\end{theorem} 

We outline the proof, providing the details in the supplementary material.
The proof shows that boundary bias of the LTSA smoother is of order $\Omega(h^2 Tr(\hcal f(x) \Sigma))$ for some matrix $\Sigma$. This is of lower order than
the required scaling $h^{-4}$ for the functional to converge. 

Since locally linear and constant functions are in the null space by construction, only the behavior of the smoother on quadratic polynomials needs to be considered. 
The difficulty in analyzing the behavior arises from needing to average over multiple local regressions on different neighborhoods and dealing with a non-product measure.
By exploiting symmetry, the problem of a non-product measure is removed, and the multivariate regression problem can be reduced to a univariate linear regression with respect to $u_1$, where $u_1$ is the coordinate function in the direction $\eta_x$ orthogonal to the boundary. Due to the shape of the asymmetric neighborhood, the quadratic functions $u_i^2$, $ i > 1$, induce a nonzero conditional mean $E(u_i^2 | u_1) = (1-u_1^2)/(m+1)$. Since $u_1^2$ has a negative sign in this expression, the coefficients for $u_1^2$ and each $u_i^2$ can be set to cancel out the boundary bias to obtain the result.

\subsection{Partial equivalence of HLLE and Laplacian}
Although the Laplace-Beltrami operator induces a penalty on the gradient and appears fundamentally different from the Hessian penalty induced by HLLE and LTSA, they can be compared by squaring the Laplace-Beltrami operator to create a fourth-order differential operator that induces a slightly different penalty on the Hessian.

Consider the smoothness penalty  obtained by squaring the Laplace-Beltrami operator compared to the HLLE penalty.
\begin{align*}
&J_{\Delta^2}(f) = \langle f, \Delta^2 f\rangle = \langle \Delta f, \Delta f \rangle  
= \int Tr(\hcal f(z))^2 dV_g(z)  \\
&J_{HLLE}(f) = \int Tr\left( (\hcal f(z))^2 \right) dV_g(z),
\end{align*}
where the second equality follows from the self-adjointness of the Laplace-Beltrami operator. These penalties are identical for one-dimensional manifolds. 
However, a slightly different penalty is induced on multidimensional manifolds. Given global coordinates $g_1, \ldots, g_m$, the quadratic polynomial $g_1 g_2$ is in the null space of the twice-iterated Laplacian but not in the null space of the HLLE penalty.

\section{Alternate constructions}
We provide a few examples to illustrate how alternative constructions can address undesirable properties of existing methods and the implications of the changes. In particular, we propose a convergent and more stable variation of LLE. We find LLE and its variants behave similarly to a local linear regression smoother, replacing the Neumann boundary condition for Laplacians with a second-order boundary condition that admits linear functions. Furthermore, we generate a new Laplace-Beltrami approximation that generates a smoothing penalty that does not impose boundary conditions. The resulting operator is shown to not have well-separated eigenfunctions. 

\subsection{Low-dimensional representation LLE+}
LLE has the deficiency that the curvature of the manifold affects the resulting operator and smoothness penalty. In the worst case when the regularization parameter is very small, the embedding $\Phi : \mcal \to \mathbb{R}^d$ of the manifold in the ambient space lies close to the null space of the resulting LLE operator. In other words, a spectral decomposition of the LLE operator may simply recover a linear transformation of the ambient space.

Low dimensional representation \cite{goldberg2008ldrlle} LLE (LDR-LLE) is a modified version of LLE that removes some of the effect of curvature in the manifold by reconstructing each point based on its tangent space coordinates only. As this still cancels out the first two terms of a Taylor expansion, it is an approximation to the Laplace-Beltrami operator. 
The weights at $X_i$ are chosen to be in the subspace spanned by points in $\ncal_i$ in the ambient space and orthogonal to the tangent space, and thus, explicitly penalizes functions with curvature that matches the manifold. However, as the singular values for directions orthogonal to the tangent space decrease as $O(h^2)$ compared to $O(h)$ for the tangent space, the method is subject to numerical instability. As such we classify LDR-LLE as well as LLE as being based on a second-order smoother. To compensate for this instability, it adds a regularization term.

We propose a further modification, LDR-LLE+, which removes the artificial restriction on the weights to the span of the points in the neighborhood and generates a first-order smoother. The weights are given by
\begin{align}
w_i &\propto 1_{\ncal_i} - U_{1:m} U_{1:m}^T 1,
\end{align}
where $U_{1:m}$ are the top left singular vectors of $X_{\ncal_i}$ after centering on $X_{i}$.

Another variant of LLE is modified LLE (mLLE) \cite{zhang2007mlle}. While LDR-LLE explicity constructs a local smoother that can be expressed by linear functions in the tangent coordinates and a bias operator in which linear functions are in its null space, mLLE achieves a similar effect by adding vectors orthogonal to the tangent space to the bias operator in order to remove non-tangent directions from the null space.

\subsection{Local Linear Regression}
A simple alternative construction is to directly use a local linear regression (LLR) smoother to approximate the Laplace-Beltrami operator. 
Since only functions linear in the normal coordinates should be included in the null space of the bias operator, each local regression is performed on the projection of the neighborhood into its tangent space. \citet{cheng2013local} propose this linear smoother in the context of local linear regression on a manifold and note its connection to the Laplace-Beltrami operator without analyzing the convergence of the associated operator. 
Our contribution is determining the boundary conditions. 

In particular, we study the boundary behavior for the operator $L^{LLR*}L^{LLR}$ where $L^{LLR*}$ is the adjoint of  
$L^{LLR}$. The resulting boundary conditions have the same form as those for LTSA, albeit with different constants, and admit linear functions as eigenfunctions. 

\begin{theorem}
	Let $S_h$ be the expected continuous extension for a local linear regression smoother with bandwidth $h$ and $L_h = h^{-2}Vol_h^{-1}(I - S_h)$.
	Then there is some $\beta > 0$ such that for any $f \in C^{\infty}(\mcal)$ and $x \in \partial \mcal$,
	$L_h^*L_h f(x)$ converges as $h \to 0$ only if	
	\begin{align}
	\beta \frac{\partial^2 f}{\partial \eta_x^2}(x) &= (\Delta f)(x),
	\end{align}
	where $\eta_x$ is the vector normal to the boundary at $x$.
\end{theorem}

This result is obtained by examining the behavior of the adjoint $L_h^*$ by reducing it to computing the average residual over a set of univariate linear regressions. The result is that the adjoint has boundary bias $\Omega(g(x))$ when applied to a function $g$. Combined with existing analyses for the boundary bias of $L^{(LLR)}$ \cite{ruppert1994multivariate} gives $(L_h^* L_h f)(x) = \Omega(h^{-2})\Omega(1) \to \infty$ if the boundary condition is not met. Interestingly, this yields a Dirichlet boundary condition when swapping the order the operators are applied $L_h L_h^*$ as $h \to 0$.

We note that this procedure is similar to LDR-LLE. The row weights must sum to one to preserve constant functions, $S_n 1 = 1$, and each point's tangent coordinates are perfectly reconstructed since the coordinate functions are linear functions of themselves. Empirically we found these perform very similarly as shown in Figure \ref{fig:swisshole}, but Local Linear Regression was more likely to exclude the quadratic interaction polynomial from its bottommost eigenfunctions.

\vspace{-0.1cm}
\subsection{Laplacian without boundary conditions}
For existing Laplacian approximations, the smoothness penalty $f^T L_n f \to \|\nabla f\|^2$ is only guaranteed to converge when $f$  satisfies the boundary conditions imposed by $L$. Furthermore, the only existing Laplace-Beltrami approximation that yields a positive semidefinite matrix and non-negative smoothness penalty is the graph Laplacian.

We derive a construction, the {\em Coefficient Laplacian}, that is both self-adjoint and guarantees convergence of the smoothness penalty for all $C^{\infty}(\mcal)$ functions. 
This is achieved by explicitly estimating the gradient with a local linear regression. For each neighborhood $\ncal_i$, estimate normal coordinates $U_{\ncal_i}$ with respect to $X_i$.
Define 
\begin{align}
S_{\ncal_i} &= (\tilde{U}_{\ncal_i}^T\tilde{U}_{\ncal_i})^{-1} \tilde{U}_{\ncal_i}^T (I - 1 e_i^T)\\
L &= \sum_i A_{\ncal_i} S_{\ncal_i}^T S_{\ncal_i} A_{\ncal_i}^T.
\end{align}
Thus, $S_{\ncal_i} Y_{\ncal_i}$ yields the coefficients for a linear regression of the centered $Y_{\ncal_i} - 1 Y_i$ on $U_{\ncal_i}$ without the intercept term. 
Thus, $f^TLf = \sum_i f^T Q_{\ncal_i} f \to \sum_i \|\nabla f\|^2$
as $n \to \infty$ and $h \to 0$ sufficiently slowly.

The resulting operator has desirable behavior when used as a smoothing penalty. However, its eigenfunctions are useless for nonlinear dimensionality reduction. This occurs for two reasons. First, removal of boundary conditions yields an operator without a pure point spectrum. 
For example, in one dimension, the resulting operator solves the differential equation $y'' = -\lambda y$ which has solutions $y = cos(\sqrt{\lambda}x)$ for all $\lambda > 0$. Thus, it is unclear what solutions are  picked out by an eigendecomposition of the discrete approximations. Second, one can show that $L_h f$ does not converge uniformly at the boundary even when the functional $f^T L_h f$ converges.
As a result, the eigenvectors of the discrete approximations are uninformative. The resulting deficiencies are illustrated in Figure \ref{fig:spectra}.

\setlength{\belowcaptionskip}{-0.2cm}
\begin{figure}
	\vspace{-0.2cm}
	\includegraphics[height=1.5in, width=3in]{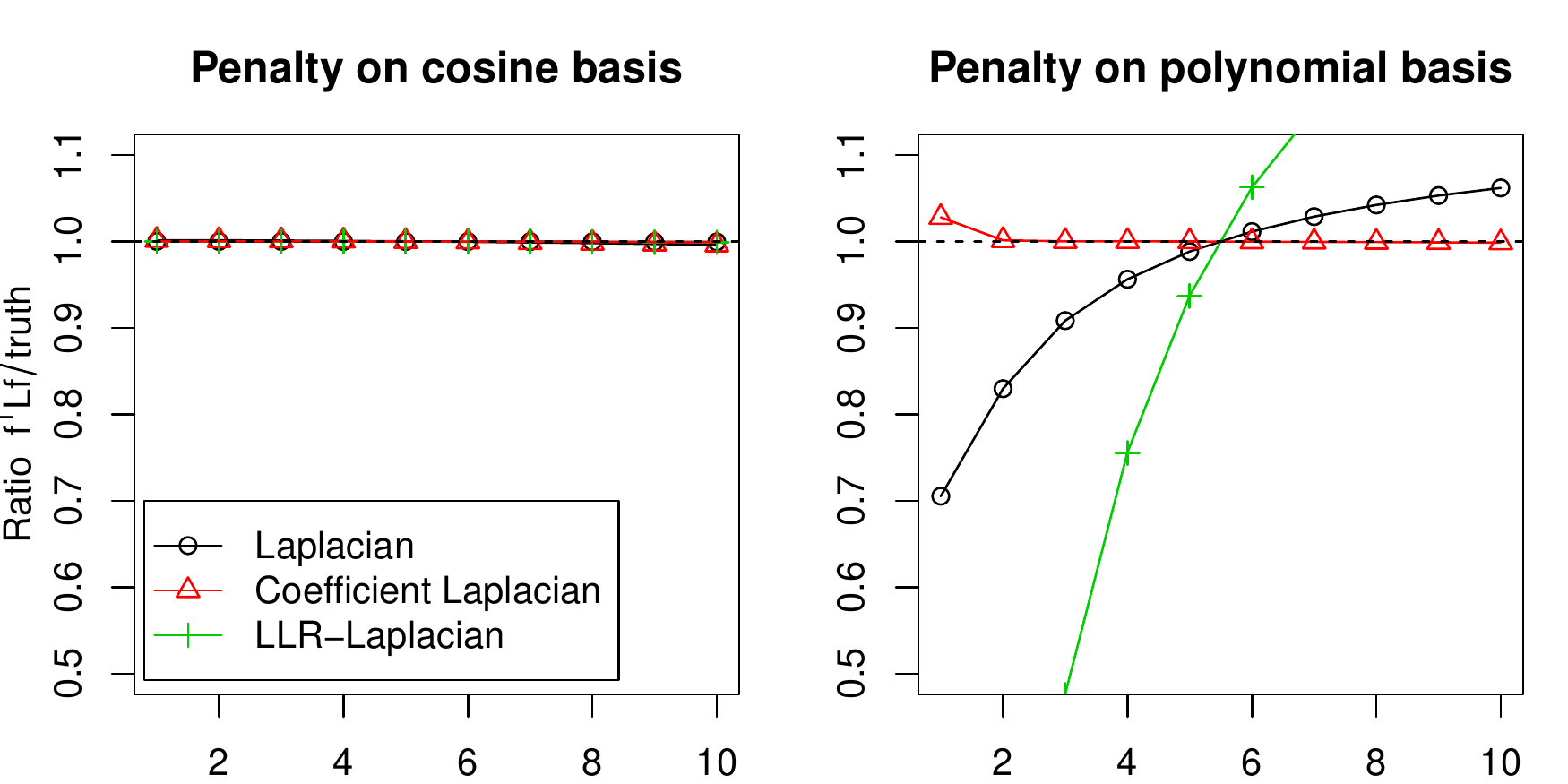} 
	\vspace{-0.2cm}
	\caption{When restricted to cosine functions which satisfy the Neumann boundary condition, the smoothness penalty induced by different Laplacian approximations are all similar. For other functions, $sign(x) |x|^{(i+1)/3}$, the smoothness penalty can significantly differ. The Coefficient Laplacian remains close to the true penalty even in this case. }
	\label{fig:smoothness}
\end{figure}

\vspace{-0.15cm}
\section{Discussion}
This paper explores only one aspect of the connection between manifold learning and linear smoothing, namely the analysis of manifold learning methods and ways to use that analysis to improve existing NLDR methods.

However, this connection can be exploited in many additional ways. For instance, the methods developed for manifold learning can be used for linear smoothing. In particular, the linear smoother associated with LTSA can be used as a smoother for non-parametric regression problems. A simple consequence of our analysis is that the smoother can achieve $O(h^4)$ bias in the interior even though it only uses a first-order local linear regression. 
By contrast, local linear regression yields a bias of $O(h^2)$ in the interior.
The connection can also be exploited for multi-scale analyses. In the case of the Laplacian, the weighted bias operator $L$ is the infinitesimal generator of a diffusion with $P_t = exp(-Lt)$ defining its transition kernels and can be used to generate Diffusion wavelets \cite{coifman2006diffusionwavelets}. Other bias operators can be substituted to generate orthogonal bases with local, compact support.
Another possibility is to combine existing linear smoothers with manifold learning penalties to generate new smoothers and NLDR techniques by using the theory of generalized additive models \cite{buja1989linear}.

\bibliography{ling2}

\begin{thebibliography}{20}
\providecommand{\natexlab}[1]{#1}
\providecommand{\url}[1]{\texttt{#1}}
\expandafter\ifx\csname urlstyle\endcsname\relax
  \providecommand{\doi}[1]{doi: #1}\else
  \providecommand{\doi}{doi: \begingroup \urlstyle{rm}\Url}\fi

\bibitem[Belkin \& Niyogi(2003)Belkin and Niyogi]{belkin2003laplacian}
Belkin, M. and Niyogi, P.
\newblock {Laplacian eigenmaps for dimensionality reduction and data
  representation}.
\newblock \emph{Neural Computation}, 15\penalty0 (6):\penalty0 1373--1396,
  2003.

\bibitem[Belkin \& Niyogi(2007)Belkin and Niyogi]{belkin2007convergence}
Belkin, M. and Niyogi, P.
\newblock Convergence of {L}aplacian eigenmaps.
\newblock In \emph{NIPS}, 2007.

\bibitem[Berry \& Sauer(2016)Berry and Sauer]{berry2016consistent}
Berry, T. and Sauer, T.
\newblock Consistent manifold representation for topological data analysis.
\newblock \emph{arXiv preprint arXiv:1606.02353}, 2016.

\bibitem[Buja et~al.(1989)Buja, Hastie, and Tibshirani]{buja1989linear}
Buja, A., Hastie, T., and Tibshirani, R.
\newblock Linear smoothers and additive models.
\newblock \emph{The Annals of Statistics}, 17\penalty0 (2):\penalty0 453--510,
  1989.

\bibitem[Cheng \& Wu(2013)Cheng and Wu]{cheng2013local}
Cheng, M.Y. and Wu, H.T.
\newblock Local linear regression on manifolds and its geometric
  interpretation.
\newblock \emph{Journal of the American Statistical Association}, 108\penalty0
  (504):\penalty0 1421--1434, 2013.

\bibitem[Coifman \& Lafon(2006)Coifman and Lafon]{coifman2006diffusion}
Coifman, R. and Lafon, S.
\newblock Diffusion maps.
\newblock \emph{Applied and Computational Harmonic Analysis}, 21\penalty0
  (1):\penalty0 5--30, 2006.

\bibitem[Coifman \& Maggioni(2006)Coifman and
  Maggioni]{coifman2006diffusionwavelets}
Coifman, R. and Maggioni, M.
\newblock Diffusion wavelets.
\newblock \emph{Applied and Computational Harmonic Analysis}, 21\penalty0
  (1):\penalty0 53--94, 2006.

\bibitem[Donoho \& Grimes(2003)Donoho and Grimes]{HessEig}
Donoho, D.~L. and Grimes, C.
\newblock {Hessian eigenmaps: Locally linear embedding techniques for
  high-dimensional data}.
\newblock \emph{Proceedings of the National Academy of Sciences}, 100\penalty0
  (10):\penalty0 5591, 2003.

\bibitem[Goldberg \& Ritov(2008)Goldberg and Ritov]{goldberg2008ldrlle}
Goldberg, Y. and Ritov, Y.
\newblock {LDR}-{LLE}: {LLE} with low-dimensional neighborhood representation.
\newblock \emph{Advances in Visual Computing}, pp.\  43--54, 2008.

\bibitem[Hein et~al.(2007)Hein, Audibert, and
  {von~Luxburg}]{HeinGraphNormalizations}
Hein, M., Audibert, J.-Y., and {von~Luxburg}, U.
\newblock Graph {L}aplacians and their convergence on random neighborhood
  graphs.
\newblock \emph{Journal of Machine Learning Research}, 8:\penalty0 1325--1370,
  2007.

\bibitem[Roweis \& Saul(2000)Roweis and Saul]{roweis00LLE}
Roweis, S.~T. and Saul, L.~K.
\newblock {Nonlinear dimensionality reduction by locally linear embedding}.
\newblock \emph{Science}, 290\penalty0 (5500):\penalty0 2323, 2000.

\bibitem[Ruppert \& Wand(1994)Ruppert and Wand]{ruppert1994multivariate}
Ruppert, D. and Wand, M.
\newblock Multivariate locally weighted least squares regression.
\newblock \emph{The Annals of Statistics}, 22\penalty0 (3):\penalty0
  1346--1370, 1994.

\bibitem[Singer \& Wu(2016)Singer and Wu]{singer2016spectral}
Singer, A. and Wu, H.T.
\newblock Spectral convergence of the connection {L}aplacian from random
  samples.
\newblock \emph{Information and Inference: A Journal of the IMA}, 6\penalty0
  (1):\penalty0 58--123, 2016.

\bibitem[Tenenbaum et~al.(2000)Tenenbaum, De~Silva, and
  Langford]{tenenbaum2000global}
Tenenbaum, J., De~Silva, V., and Langford, J.
\newblock A global geometric framework for nonlinear dimensionality reduction.
\newblock \emph{Science}, 290\penalty0 (5500):\penalty0 2319--2323, 2000.

\bibitem[Ting et~al.(2010)Ting, Huang, and Jordan]{ting2010laplacian}
Ting, D., Huang, L., and Jordan, M.~I.
\newblock An analysis of the convergence of graph {L}aplacians.
\newblock In \emph{ICML}, 2010.

\bibitem[{von~Luxburg} et~al.(2008){von~Luxburg}, Belkin, and
  Bousquet]{ConsistencySC}
{von~Luxburg}, U., Belkin, M., and Bousquet, O.
\newblock Consistency of spectral clustering.
\newblock \emph{Annals of Statistics}, 36\penalty0 (2):\penalty0 555--586,
  2008.

\bibitem[Weinberger \& Saul(2006)Weinberger and
  Saul]{weinberger2006unsupervised}
Weinberger, K. and Saul, L.
\newblock Unsupervised learning of image manifolds by semidefinite programming.
\newblock \emph{International journal of computer vision}, 70\penalty0
  (1):\penalty0 77--90, 2006.

\bibitem[Zhang et~al.(2018)Zhang, Ma, and Tan]{zhang2018hlleltsa}
Zhang, S., Ma, Z., and Tan, H.
\newblock On the equivalence of {HLLE} and {LTSA}.
\newblock \emph{IEEE transactions on cybernetics}, 48\penalty0 (2):\penalty0
  742--753, 2018.

\bibitem[Zhang \& Wang(2007)Zhang and Wang]{zhang2007mlle}
Zhang, Z. and Wang, J.
\newblock {MLLE}: Modified locally linear embedding using multiple weights.
\newblock In \emph{NIPS}, 2007.

\bibitem[Zhang \& Zha(2004)Zhang and Zha]{zhang2004principal}
Zhang, Z. and Zha, H.
\newblock Principal manifolds and nonlinear dimensionality reduction via
  tangent space alignment.
\newblock \emph{SIAM journal on scientific computing}, 26\penalty0
  (1):\penalty0 313--338, 2004.

\end{thebibliography}
\bibliographystyle{icml2018}

\onecolumn
\appendix
\appendixpage

\section{Convergence to differential operator}
\begin{theorem}
	Let $S_n$ be a sequence of linear smoothers where the support of $S_n(x,\cdot)$ is contained in a ball of radius $h_n$.
	Further assume that the bias in the residuals $h_n^{-k}(I-S_n)f = o(1)$ for some integer $k$
	and all $f \in \ccal \subset C^{\infty}(\mcal)$. Then if $h_n \to 0$ and $h_n^{-k}(I-S_n)$ converges to a bounded linear operator as $n \to \infty$, then it must converge to a differential operator of order at most $k$ on the domain $\ccal$.
\end{theorem}

\begin{proof}
	Let $S_\infty$ is be the limit of $h^{-k}(I-S_n)$. 
	Let $q_\ell$ be any monomial with  degree $\ell > k$ and $I_{x,n}$ be some smooth function that is $1$ on a ball of radius $h_n$ around $x$.
	Then the pointwise product $\| q_\ell \cdot I_{x,n}\| = O(h_n^\ell)$.
	Since $S_\infty$ is bounded, $(S_n q_\ell)(x) = (S_n (I_{x,n} \cdot q_\ell))(x) = O(h_n^{\ell - k}) \to 0$ . 
	Furthermore, the convergence is uniform over all $x$.
	Thus, the behavior is determined on a basis of polynomials of degree at most $k$. 
	Applying $S_\infty$ to a Taylor expansion with remainder gives the desired result.
\end{proof}

\section{Equivalence of HLLE and LTSA}
\begin{theorem}
	$L_h^{LTSA} - L_h^{HLLE} \to 0$ as $h\to 0$ in the weak operator topology of $C^\infty(\mcal)$ equipped with the $L_2$ norm.
\end{theorem}
\begin{proof}
	The local operators $Q^{HLLE}_i$ for HLLE 
	can be described more succinctly in 
	terms of the difference of two linear smoothers.
	Let $H_x^{(2)}$ be hat matrix for a local quadratic regression in the neighborhood $\ncal_x$. 
	\begin{align}
	Q^{HLLE}_x &= \alpha_h (H^{(2)}_x -  H^{(1)}_x) \\
	&= \alpha_h (Q^{LTSA}_x + (H^{(2)}_x - I_{\ncal_x})).
	\end{align}
	Let $R_x = H^{(2)}_x - I_{\ncal_x}$. 
	For any $f,g \in C^{\infty}(\mcal)$, 
	\begin{align}
	\alpha_h g^T R_x f &= \alpha_h \langle R_x g, R_x f\rangle
	\leq O(h^{-4+3+3})
	\end{align}
	since quadratic terms and lower can be removed from the Taylor expansion.
\end{proof}

\section{Boundary behavior of HLLE and LTSA}

\begin{theorem}
	\label{thm:HLLE boundary}
	For any $x \in \partial \mcal$ and function $f \in C^{\infty}(\mcal)$, $L^{LTSA}_h f(x) \to \infty$ as $h\to 0$ unless
	\begin{align}
	\label{eqn:HLLE boundary}
	\frac{\partial^2 f}{\partial \eta_x^2}(x) &= \frac{m+1}{m+2} (\Delta f)(x) \quad \forall x \in \partial \mcal
	\end{align}
	where $\eta_x$ is the tangent vector orthogonal to the boundary.
\end{theorem} 

To do this, we show that the required scaling by $h^{-4}$ for the functional to converge causes the value at the boundary to go to $\infty$ when this condition is not met.

Let $x \in \mcal$. Assume $h$ is sufficiently small so that there exists a normal coordinate chart at $x$ such that $B(x, 2h)$ is contained in the neighborhood for the chart. 
Furthermore, choose the normal coordinates such that the first coordinate $u_1$ corresponds to the direction $\eta_x$ orthogonal to the boundary and pointing inwards.

Consider the behavior of $L_h$ on the basis of polynomials expressed in this coordinate chart.
Denote by $Q_y$ the hat operator for a linear regression on linear and constant functions restricted to a ball of radius $h$ centered on $y$.
Note that any polynomial of degree at most 1 is in the null space of any $Q_y$ by construction. 
Likewise, by symmetry, any polynomial $u_i u_j$ with $j < i$ or of odd degree is in the null space of $Q_y$. This leaves only the polynomials $u_i^2$. 

Each $(Q_y f)(z)$ computes the residual at $z$ from regressing $f$ against linear functions in the neighborhood at $y$.
The residual $r_i = (Q_y u_i^2)(0) = (Q_y (u_i - y_i)^2)(0)$ for all $y \in B(0,h)$
since linear functions are in the null space.
By symmetry of the neighborhoods in the interior,
$r_y = (Q_{0} u_i^2)(-y)$ whenever $i > 1$.
In other words, in the interior of the manifold, averaging over the multiple regressions is equivalent to averaging over the residuals of a single regression. 
Since the expected residual is always $0$ for least squares regression,
$(L_h f)(x) = O(h^4)$ for any $x \in Int(\mcal)$.

More intuitively, one wishes to evaluate the  residual at $z$ over all neighborhoods $\ncal_y$ that include $z$. 
Because linear functions are in the null space, this depends only on the curvature of the response $f$ at $x$.
Furthermore, since all neighborhoods have essentially identical shape, the hat matrices for every neighborhood are also identical,
so the residual only depends on the offset $z - y$. Averaging over the different $y$ is equivalent to fixing a single regression at $B(z,h)$ and averaging over the residuals at all offsets $y \in B(z,h)$.

However, the boundary behavior is very different since the shape of the neighborhoods can change. 
First, consider the quadratic polynomial $u_1^2$ in the direction orthogonal to the tangent space on the boundary. 
Since the residual is always evaluated at the boundary for a linear regression and $u_1^2$ is convex, the residual is necessarily positive. Figure \ref{fig:convex residual} provides an intuitive illustration of this.
Furthermore, since evaluating 
$(Q^{(\alpha h)}_y (\alpha h)^{-2} u_1^2)(x) = O( (\alpha h)^{-4})$
It follows that $(L_h f)(x) =  \alpha h^2 \partial^2 f / \partial u_1^2(x) + O(h^3)$ for some non-zero constant $\alpha$.

For polynomials $u_i^2$ where $i > 1$, we note that any linear 
regression on linear functions $u_j$ has corresponding coefficients $\beta_j = 0$ for all $j > 1$ by symmetry of the neighborhood. 
Thus, the prediction can be reduced to a linear regression where the response is the conditional mean $E(u_i^2 | u_1)$ given the neighborhood. Integrating over slices of an $m$-sphere gives
\begin{align}
E(u_i^2 | u_1) = (m+1)^{-1}(1-u_1^2).
\end{align}
Since constant terms are removed by the regression and since summing over the residual function for a given value of $u_1$ gives $\E (u_i^2 - \hat{u}_i^2 | u_1) = 
\E (u_i^2 -  \E(u_i^2 | u_1) | u_1)  + \E(\E(u_i^2 | u_1) - \hat{u}_i^2 | u_1) = 
\E(\E(u_i^2 | u_1) - \hat{u}_i^2 | u_1)$, 
the summed residuals are the same as those obtained for the function $-u_1^2 / (m+1)$. 

Thus, for any $x \in \partial \mcal$, $(L_h f)(x) \to \infty$ unless $0 = \partial^2 f / \partial u_1^2 + \frac{m+1}{m-1} \sum_{i=2}^m \partial^2 f / \partial u_i^2$ which is equivalent to equation \ref{eqn:HLLE boundary}.

\begin{figure}[H]
	\includegraphics[width=4in]{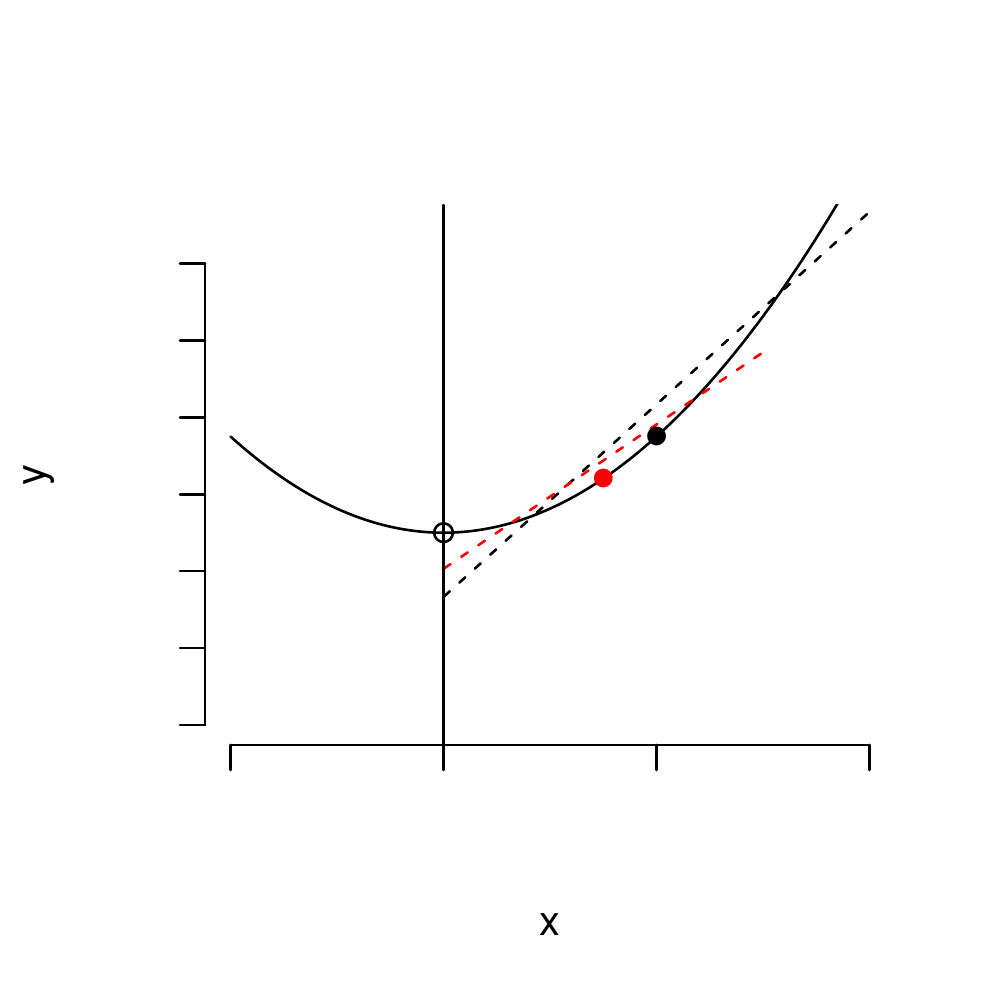}
	\caption{Figure illustrating that the residual is always positive when evaluating a linear regression on the boundary.
The dots are the centers defining the neighborhood on which regression is performed. 		The vertical line represents the boundary of the manifold that the neighborhoods do not cross. The dashed lines are the regression fits. It is easy to see that the residuals at the boundary are always strictly positive.}
\label{fig:convex residual}
\end{figure}

We checked this boundary condition using simulation on a manifold isomorphic to a rectangle.
We took 10 estimated eigenfunctions and computed their Hessian at a point on the boundary.
This generates a $10 \times 3$ matrix consisting of the estimates $\frac{\partial^2 f}{\partial u_1^2}, \frac{\partial^2 f}{\partial u_2^2}, \frac{\partial^2 f}{\partial u_1\partial u_2}$.
We take the svd of this matrix. The distribution of the top, middle, and bottom singular value is shown in figure \ref{fig:eigenfun svd}. There is one eigenvalue clearly close to 0 that represents the boundary condition.
The average bottom right singular vector is given in table \ref{tbl:eigenfun hessian values} and show fairly good correspondence to the theoretical calculations on a modestly fine grid.

\begin{figure}[H]
	\includegraphics[width=4in]{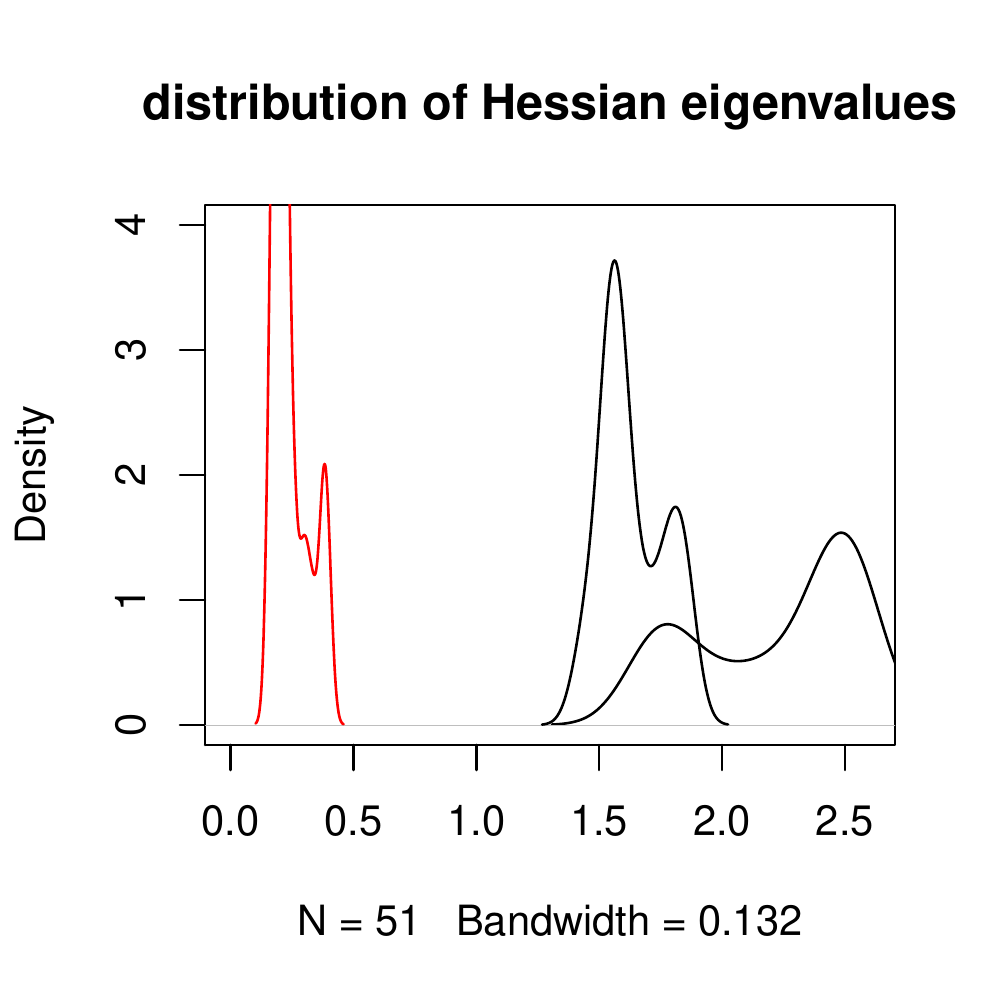}
	\caption{Figure illustrating singular values of the Hessian of eigenfunctions at the boundary.}
	\label{fig:eigenfun svd}
\end{figure}

\begin{table}[H]
	\begin{tabular}{c|c} 
		$2^{nd}$ Derivative & Estimated value \\ \hline
		$\frac{\partial^2 f}{\partial u_1^2}$ & 0.935 \\
		$\frac{\partial^2 f}{\partial u_2^2}$ & -0.346 \\
		$ \frac{\partial^2 f}{\partial u_1\partial u_2}$ & 0.00 \\
		$\Delta f$ & 0.588 \\
		Predicted $\frac{\partial^2 f}{\partial u_1^2}$ & 0.882
		\end{tabular}
	\caption{Table of mean Hessian values and the prediction of the $2^{nd}$ derivative given the Laplacian showing fair correspondence of the simulated values to the predicted ones on a modestly fine grid.}
	\label{tbl:eigenfun hessian values}
\end{table}

\section{Boundary Behavior of Local linear regression}

As the boundary bias of local linear regression is already well studied, existing results can help determine the boundary conditions for the resulting operator. However, as the local linear regression smoother is not self-adjoint, the behavior of the adjoint must also be determined.
Let $x \in \partial \mcal$ and $\mathbf{u}$ be normal coordinates at $x$ in a neighborhood $\bcal$ containing $B(x, 2h) \cap \mcal$. We again choose $u_1$ to correspond to the tangent vector that is normal to the boundary and pointing inwards so that $u_1 \geq 0$.
We wish to evaluate the boundary behavior
of $L^*_h L_h f(x) = \langle L_h \delta_x, L_h f \rangle$
where $L^*_h$ is the adjoint of $L_h$ and $\delta_x$ is the Dirac delta.
The main part of our proof is to evaluate $L_h \delta_x$.

Similar to the proof for theorem \ref{thm:HLLE boundary}, we first show we can reduce the problem to a univariate linear regression. We do this through an orthogonal basis. Then we use the usual regression equations to actually compute the value.
For each $y \in B(x,h) \cap \mcal$,
denote $\ncal_y = B(y,h) \cap \mcal$ and 
$\tilde{u}^{(0)}_i(\cdot, y) = h^{-1}Vol(\ncal_y)^{-1} (u_i(\cdot)1(\cdot \in \ncal_y) - u_i(y))$.
It is easy to see that by symmetry that there are functions $ \tilde{u}_i(\cdot,y) = \tilde{u}^{(0)}_i(\cdot, y) + \epsilon_i$ which are orthogonal to each other and to the constant function in $L_2(\ncal_y)$ for $i > 1$ and where $\epsilon_i = O(h^2)$. 
Likewise, let $\tilde{u}_1(\cdot, y)$ be similarly defined to be orthogonal to $\tilde{u}_i(\cdot, y)$ for $i > 1$. To generate an orthonormal basis, we must orthogonalize it with respect to the constant function as well. Gram-Schmidt gives $ v_1(\cdot, y) = \tilde{u}_1(\cdot, y)- Vol(\ncal_y)^{-1} \langle \tilde{u}_1(\cdot, y), I(\cdot \in \ncal_y) \rangle I(\cdot \in \ncal_y)$ is  orthogonal to the constant function.
Thus, by evaluating 
\begin{align}
h^2 (L_h \delta_x)(y) &=  \delta_x(y) - Vol(\ncal_y)^{-1} I(x \in \ncal_y) \\
&\quad - \sum_{i=2}^m \|\tilde{u}_i(\cdot, y)\|^{-2} \tilde{u}_i(y,y) \langle \tilde{u}_i(\cdot, y), \delta_x \rangle \\
&\quad - \|v_1(\cdot, y) \|^{-2} v_1(y, y) \langle v_1(\cdot, y), \delta_x \rangle  \\
&= \delta_x(y) - Vol(\ncal_y)^{-1} I(y \in \ncal_x) \\ 
&\quad - Vol(\ncal_y)^{-1}\|v_1(\cdot, y) \|^{-2} \langle \tilde{u}_1(\cdot, y), I(\cdot \in \ncal_y) \rangle v_1(x,y) 
\end{align}
where the inner product is taken with respect to $L_2(\bcal)$. 
Integrating over $\bcal$, the first two terms each have magnitude 1 and cancel each other out. 
The third term may be rewritten
\begin{align}
&- Vol(\ncal_y)^{-1}\|v_1(\cdot, y) \|^{-2} \langle \tilde{u}_1(\cdot, y), I(\cdot \in \ncal_y) \rangle v_1(x,y) \\
&\quad = - \|v_1(\cdot, y) \|^{-2} (\tilde{u}_1(x,y) \mu_y  - \mu_y^2) \\
&\quad =  \|v_1(\cdot, y) \|^{-2} (Vol(\ncal_y)^{-1}h^{-1}u_1(y) \mu_y  + \mu_y^2) \\
\end{align}
where $\mu_y = Vol(\ncal_y)^{-1} \langle \tilde{u}_1(\cdot, y), I(\cdot \in \ncal_y) \rangle > 0$ unless $u_1(y)(1+O(h^2)) \geq h $.

It follows that
\begin{align}
\langle L_h \delta_x,  I(\cdot \in \ncal_x) \rangle &= 
h^{-2}  \langle \Omega(Vol(\ncal_y)^{-1}), I(\cdot \in \ncal_x) \rangle \\
&= \Omega(h^{-2}) \\
&\to \infty \quad \mbox{as $h \to 0$}
\end{align}
By applying this to a Taylor expansion, one concludes that for any continuously differentiable function which is non-zero at $x \in \partial \mcal$, $|(L_h^* f)(x)| \to \infty$ as $h \to 0$.

By Theorem 2.2 in \cite{ruppert1994multivariate}, a point $x$ with distance $dist(x, \partial \mcal) < h$ has boundary bias 
\begin{align}
(L_h f)(x) &= \alpha (\mu_2 \,\, \mu_1)
\int_{D_{x,h}} \left(
\begin{array}{c} 1 \\ u_1 \end{array}
\right) Tr(\hcal f(x) \mathbf{u} \mathbf{u}^T) d\mathbf{u} \nonumber \\
\mu_1 &= \int_{D_{x,h}} u_1 d\mathbf{u}  \nonumber \\
\mu_2 &= \int_{D_{x,h}} u_1^2 d\mathbf{u}  + o(Trace(\hcal f(x)))  \nonumber 
\end{align}
where $\alpha > 0$ is some constant and $D_{x,h}$ is the unit $m$-sphere cut along the plane orthogonal to the first coordinate at $u_1 = dist(x, \partial \mcal) / h$. 
This is a linear function in the Hessian. Furthermore, all odd moments of $u_i$ for $i > 1$ are 0, and their second moments are all equal. It is a linear function of the diagonal of the Hessian and of the form
$(L_h f)(x) = -\beta \frac{\partial^2 f}{\partial u_1^2}(x) + (\Delta f)(x) + o(1)$ for some $\beta \neq 0$.

Thus, the boundary condition that functions must satisfy is
\begin{align}
\beta \frac{\partial^2 f}{\partial u_1^2}(x) &= (\Delta f)(x).
\end{align}

Although it takes the same form as the HLLE boundary condition, the constants are different. However, we do not know of a reason to prefer one boundary condition over the other.

\end{document}